\documentclass[letterpaper, 10 pt, conference]{ieeeconf}  

\IEEEoverridecommandlockouts                              

\usepackage{amsmath} 
\usepackage{amssymb}  

\usepackage{amsthm}
\usepackage{graphicx}
\usepackage{url} 
\usepackage{float}
\newtheorem{theorem}{Theorem}[section]

\newtheorem{conjecture}[theorem]{Conjecture}
\newtheorem{note}[theorem]{Note}

\newtheorem{definition}{Definition}[section]

\title{\LARGE \bf
	A New Similarity Function for Spectral Clustering with Application to Plant Phenotypic Data
}

\author{Kapil Ahuja$^{1}$, Mithun Singh$^{1}$, Kuldeep Pathak$^{1}$,  Milind B. Ratnaparkhe$^{2}$ 
	\thanks{*This work was not supported by any organization}
	\thanks{$^{1}$Math of Data Science \& Simulation (MODSS) Lab, Computer Science \& Engineering
		Indian Institute of Technology Indore, Indore, India {\tt\small kahuja@iiti.ac.in}}%
	\thanks{$^{2}$ICAR-Indian Institute of Soybean Research, Indore, India }%
}

\begin{document}

	\maketitle
	\thispagestyle{empty}
	\pagestyle{empty}

	\begin{abstract}		
	Clustering species of the same plant into different groups is an important step in developing new species of the concerned plant. Phenotypic (or physical) characteristics of plant species are commonly used to perform clustering. Hierarchical Clustering (HC) is popularly used for this task, and this algorithm suffers from low accuracy. In one of the recent works \cite{shastri2021probabilistically}, the authors have used the standard Spectral Clustering (SC) algorithm to improve the clustering accuracy. They have demonstrated the efficacy of their algorithm on soybean species.

	In the SC algorithm, one of the crucial steps is building the similarity matrix. A Gaussian similarity function is the standard choice to build this matrix. In the past, many works have proposed variants of the Gaussian similarity function to improve the performance of the SC algorithm, however, all have focused on the variance or scaling of the Gaussian. None of the past works have investigated upon the choice of base ``e'' (Euler's number) of the Gaussian similarity function (natural exponential function). 
	
	Based upon spectral graph theory, specifically the Cheeger's inequality, in this work we propose use of a base ``$a$'' exponential function as the similarity function. We also integrate this new approach with the notion of ``local scaling'' from one of the first works that experimented with the scaling of the Gaussian similarity function \cite{zelnik2004self}.
		
	Using an eigenvalue analysis, we theoretically justify that our proposed algorithm should work better than the existing one. With evaluation on $2376$ soybean species and $1865$ rice species, we experimentally demonstrate that our new SC is 35\% and 11\% better than the standard SC, respectively.
	\end{abstract}

	\section{INTRODUCTION}
	
	Phenotypic characteristics (or physical characteristics) of plant species are often used in clustering them into separate categories \cite{painkra2018clustering, sharma2014assessing}.  This is done so that plant species from different categories (or diverse plant species) could be selectively chosen for developing new species having better characteristics \cite{swarup2021genetic} (or called breeding). Hierarchical Clustering (HC) is one of the most commonly used clustering algorithms in this domain \cite{kahraman2014cluster}. This algorithm suffers from low accuracy issues. 
	
	In one of the recent works \cite{shastri2021probabilistically}, authors have used the standard Spectral Clustering (SC), considered to be one of the most accurate clustering algorithms, for plant phenotypic data and demonstrated improved accuracy. {\it In this work, we propose new variants of the SC algorithm and demonstrate that they perform substantially better than the earlier work}. 
	
	There are four main steps in the SC algorithm; (a) capturing of relationship between different data points using a similarity matrix, (b) calculation of a Laplacian matrix from the similarity matrix, (c) computing of eigenvectors of the Laplacian matrix, and (d) use of $k-$means algorithm on the computed eigenvectors to perform clustering. 
	
	In almost all works that have used the SC algorithm, a Gaussian function has been used to build the similarity matrix. Multiple variants of this Gaussian similarity function have also been proposed to improve the accuracy of SC \cite{zelnik2004self, park2018spectral, zhang2020spectral, favati2020construction}. The focus in all such works has been in changing the variance or scaling of the Gaussian. We have a two fold contribution here. 
	
	\begin{itemize}
		\item In this work, we change the base ``$e$'' (Euler's number) of the Gaussian similarity function (natural exponential function). We propose use of a base ``$a$'' exponential function as the similarity function. Using Cheeger's inequality that originates from spectral graph theory, we prove that for a simpler Laplacian matrix if ``$a$'' is greater than ``$e$'' that this would lead to better clustering. For a more practical Laplacian matrix, although we only conjecture this result (and not prove it), we do support this choice via analysis and experiments.    
		
		\item We also integrate our above new approach with the ``local scaling'' of the Gaussian similarity function from \cite{zelnik2004self}, which was the first work to focus on scaling of the Gaussian.
	\end{itemize}
	
	We justify our clustering choices as above with an eigenvalue analysis and extensive experiments on $2376$ soybean and $1865$ rice species. 
	\begin{itemize}
		\item We show that for soybean, although the standard SC is about $32.15$\% better than HC, our base “a” SC and base “a” locally scaled SC are $72.74$\% and $81.40$\% better than HC, respectively. In other words, our best SC is ${\bf 35\%}$ better than the standard SC.
		\item We also show that for rice, although standard SC is about $49.86$\% better than HC, our base “a” SC and base “a” locally scaled SC are $64.93$\% and $66.33$\% better than HC, respectively. In other words, our best SC is ${\bf 11}${\bf\%} better than the standard SC.   
	\end{itemize}

	The rest of the manuscript has five sections. Section \ref{sec:stdSC} gives the background. In Section \ref{sec:methods}, we delve into the methods used. Section \ref{sec: analysis} gives analysis. In Section \ref{sec:results}, we give results. Finally, Section \ref{sec:conclusion} gives the conclusion.
	
	\section{Background}\label{sec:stdSC}
	SC is one of the most popular modern clustering algorithms. It is simple to implement and can be solved efficiently by standard linear algebra software. Given a set of points $S =\{p_1, p_2, ...,p_n\}$ in $R^m$ that we want to cluster into $k$ subsets, the algorithm consists of below steps \cite{von2007tutorial}. This is the algorithm that has been used in the earlier work that we extend  \cite{shastri2021probabilistically}.
	
	\begin{itemize}
		\item  Form a similarity matrix A such that
		\begin{equation}\label{eq:similarity}
			A_i{}_j= e^{\left(-\frac{d_{p_ip_j}}{2\sigma^2}\right)},
		\end{equation}
		with $i$, $j$ $\in$ $\{1, ...,n\}$ and $A_{ii}=0$. Here, $d_{p_{i}p_{j}}$ denotes the distance between two points $p_i$ and $p_j$ and $\sigma$ defines the decay of the distance.
		\item Construct the normalized Laplacian matrix 
		\begin{equation}
			L=I- D^{-\frac{1}{2}}AD^{-\frac{1}{2}},
			\label{equation:laplacian}
		\end{equation}
		where $D$ is a diagonal matrix whose $(i,i)$ element is the sum of the elements of $A$'s  $i^t{}^h$ row.
		\item Let $e_1$, $e_2$ .., $e_k$ be the first $k$ eigenvectors of $L$. Then, form the matrix $X=[e_1,e_2 ....,e_k]$ by stacking the eigenvectors as columns of this matrix.
		\item  Form $Y$ by normalizing $X$'s rows to unit length, and then Cluster $Y$ using the $k-$Means clustering.
	\end{itemize}
	
	There are many ways to the distance between points $p_i$ and $p_j$ in ($\ref{eq:similarity}$), i.e., $d_{p_{i}p_{j}}$. Some common ones are Euclidean, Squared-Euclidean, and Correlation, which are given below.
	\begin{itemize}
		\item \textbf{Euclidean}: It represents the straight-line distance between two points in Euclidean space, and is calculated as follows:
		\begin{equation}
			d_{ij}= \sqrt{\sum_{l=1}^{m}(p_{i}^{l}-p_{j}^{l})^2},
		\end{equation}
		where $p_{i}^{l}$ and $p_{j}^{l}$ are the $l^{th}$ components of $p_i$ and $p_j$ data points.
		\item \textbf{Squared-Euclidean}: It is the square of the Euclidean distance, and is given as follows:
		\begin{equation}
			d_{ij}= \sum_{l=1}^{m}(p_{i}^{l}-p_{j}^{l})^2,
		\end{equation}
		with $p_{i}^{l}$ and $p_{j}^{l}$ are defined as above.
		\item \textbf{Correlation}: It captures the correlation between two non-zero vectors, and is expressed as follows:
		\begin{equation}
			d_{ij}= 1 - \frac{(p_i -\Bar{p_i})^t(p_j -\Bar{p_j})}{\sqrt{(p_i -\Bar{p_i})^t(p_i -\Bar{p_i})}\sqrt{(p_j -\Bar{p_j})^t(p_j -\Bar{p_j})}},
		\end{equation}
		where $\Bar{p_i}$ and $\Bar{p_j}$ represent the means of vectors $p_i$ and $p_j$, respectively, multiplied by a vector of ones, and the $t$ indicates the transpose operation.
	\end{itemize}
	
	\section{Methods}\label{sec:methods}
	
	Section \ref{sec:NewClusteringAlgo} introduces a novel modification to the standard SC, which involves using a base ``$a$'' exponential function, instead of the natural exponential function, to build the similarity matrix. We theoretically justify this choice as well. In Section \ref{sec:NewestClusteringAlgo}, we combine our above novelty with another improvement of local scaling in the SC algorithm.

	\subsection{Base ``$a$'' Spectral Clustering}\label{sec:NewClusteringAlgo}
	SC is based on spectral graph theory. To derive our new algorithm, we first revisit a few concepts from this domain. We form a graph from the given data as follows \cite{gharan2020cheeger}: (a) use data points as vertices and, (b) connect each point with the remaining points with an edge having weight equal to the corresponding element of similarity matrix A.
	\begin{definition}[Conductance \cite{gharan2015cheeger}]
		Given a graph $G =(V,E)$ with $V$ partitioned into $S$ and $\overline{S}$, the conductance of $S$ is defined as
		\begin{equation} 
			\phi(S) = \frac{|E(S,\overline{S})|}{Vol(S)},
		\end{equation}
		where numerator is the fraction of edges in $cut(S,\overline{S})$ and denominator is the sum of vertices in $S$. The conductance of $G$ is defined as 
		\begin{equation}
			\phi(G)= \min_{\substack{\text{vol}(S) \leq \frac{\text{vol}(V)}{2}}}(\phi(S)),
		\end{equation}
		or the smallest conductance among all sets with at most half of the total volume.
	\end{definition}
	\par \begin{theorem}[Cheeger's Inequality \cite{gharan2015cheeger}]
		For any graph $G$,
		\begin{equation}\label{Cheeger}
			\frac{\lambda_{2}}{2} \leq \phi(G) \leq \sqrt{2\lambda_{2}},
		\end{equation}
		where $\lambda_{2}$ is the $2^{nd}$ smallest eigenvalue of $L$ given by ($\ref{equation:laplacian}$).
	\end{theorem} 
	From the above theorem, we infer that $\phi(G)$ is close to zero (or G can be grouped into  $2$ clusters) if and only if $\lambda_{2}$ is close to zero. Note that $\lambda_{1}$ is always zero. This characterization carries over to higher multiplicities as well. $G$ can be grouped into $k$ clusters if and only if there are $k$ eigenvalues close to zero \cite{lee2014multiway}. 
	\par We propose using a base ``$a$'' exponential function instead of the natural exponential function in ($\ref{eq:similarity}$) of the standard spectral clustering algorithm. That is,
	\begin{equation}
		A_{ij} = a^{\left(-\frac{d_{p_ip_j}}{2\sigma^2}\right)},
		\label{eq:pow_fun}
	\end{equation}
	where ``$a$'' $>$ ``$e$''. This results in $A_{ij}$ of ($\ref{eq:pow_fun}$) being smaller than $A_{ij}$ of ($\ref{eq:similarity}$).
	
	\par \begin{theorem} The elements of non-normalized Laplacian matrix $L=D-A$ get smaller in absolute sense when we use ($\ref{eq:pow_fun}$) instead of ($\ref{eq:similarity}$), with ``$a$'' $>$ ``$e$'', to build A. Here, $D$ is the diagonal matrix whose $(i,i)$ element is the sum of $i^{th}$ row of $A$. Further, this leads to reduction in upper bound of eigenvalues of L.
	\end{theorem} 
	
	\begin{proof}
		The first part of the Theorem is obvious. Since elements of $A$ get smaller with the proposed change of base, the elements of $D$ also get smaller ($D$ is formed via elements of A). Thus, elements of $D-A$ or $L$ get smaller in the absolute sense.
		For the second part of the proof, we use the fact that the spectral radius of the matrix is bounded above by its norm or $\rho(L) \leq ||L|| $.
	\end{proof}
	
	\begin{conjecture}\label{conjecture1}
		The above theorem holds true when we change the non-normalized Laplacian matrix $L= D-A$ with the normalized Laplacian matrix $L=I- D^{-\frac{1}{2}}AD^{-\frac{1}{2}}$.
	\end{conjecture}
	We are unable to prove this theoretically. However, this holds true experimentally. We demonstrate in the analysis section later in this paper that the change of the base as discussed in the above conjecture leads to a reduction in the eigenvalues of $L$. 
	
	Thus, from the Cheegers's Inequality (\ref{Cheeger}), we infer that we should get a better clustering when we use base ``$a$'' exponential function instead of the natural exponential function in building the similarity matrix (with ``$a$'' greater than ``$e$''). This is supported by experiments in the results section.
	
	\begin{note}
		From Fig. 1 we can see that the function value decreases exponentially when we go from $3^{-x}$ to  $3000^{-x}$. Therefore, if we continue to increase the base value of ``a'' in the above discussion infinitely, then the value of elements in the similarity matrix $A$ will tend to decrease very slowly. Hence, if the base value ``a'' is increased indefinitely, the quality of clustering will have infinitesimally small improvement.  
	\end{note}
	
	\begin{figure}[h!]
		\centering
		\includegraphics[width=.9\linewidth]{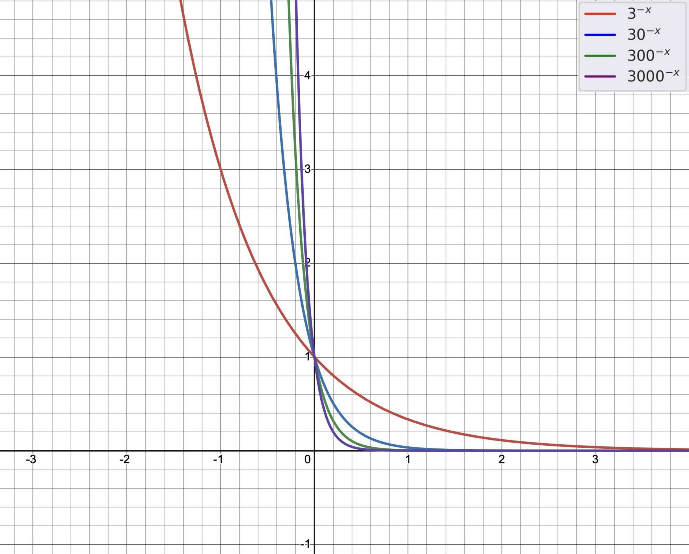}
		\caption{Exponential decay of $a^{-x}$ for different base values.}
		\label{fig:note}
	\end{figure}
	
	\subsection{Base ``$a$'' Locally Scaled Spectral Clustering}\label{sec:NewestClusteringAlgo}
	Next, to further improve our clustering, we depart from the conventional practice of utilizing a global scaling factor ($\sigma$) in (\ref{eq:pow_fun}). Instead, we adopt the concept of a local scaling factor specific to each data point, as proposed by \cite{zelnik2004self}. Now, the similarity between the two points is defined as
	\begin{equation}
		A_{ij}= a^{\left(-\frac{d_{p_{i}p_{j}}}{\sigma_i \sigma_j}\right)}.
		\label{eq: local sim}
	\end{equation}
	
	The determination of the local scale $\sigma_i$ involves analyzing the local statistics within the neighborhood of a given point. We employ a simple yet effective approach for scale selection. That is,
	\begin{equation}
		\sigma_i = d_{p_i p_K},
	\end{equation}
	where $p_K$ is the $K^\text{th}$ neighbor of $p_i$. The selection of K is independent of the scale and based upon the data dimensionality. 
	
	In the analysis section, we show that this choice of similarity function leads to further reduction in eigenvalues of $L$ (more than just use of base ``a'' exponential function). 
	
	Thus, again by Cheegers's Inequality (\ref{Cheeger}), this choice of the similarity function should lead to better clustering than both the natural exponential function and base ``a'' exponential function clustering. This is again supported by experiments in the results section.

	\section{Analysis}\label{sec: analysis}
	Few settings of our algorithms from previous sections are as follows:  (a) The best value of ``$a$'' (the base of the exponential function used to build the similarity matrix) for us turns to be ``$30$''. (b) The most fitting value of $K$ (neighbor of a point in local scaling) comes to $180$.
	
	Below, in Section \ref{subsec:Data} we describe the plant phenotypic data we test upon, i.e., for soybean and rice. Section \ref{subsec:Normalize} discusses the normalization of data. Finally, in Section \ref{subsec:EigenvalueAnalysis} we do eigenvalue analysis to justify the use of base ``$30$'' as well as local scaling in SC.
	
	\subsection{Data Description}\label{subsec:Data}
	As mentioned in the introduction, our technique is applicable to any plant dataset. However, here we focus on phenotypic data from soybean and rice species. The soybean dataset, sourced from Indian Institute of Soybean Research, Indore, India, consists of 29 different phenotypic (or physical) traits for 2376 soybean species \cite{Gireesh_Husain_Shivakumar_Satpute_Kumawat_Arya_Agarwal_Bhatia_2017}. Among these, we consider the following eight traits that are most important for higher yield: Early Plant Vigor (EPV), Plant Height (PH), Number of Primary Branches (NPB), Lodging Score (LS), Number of Pods Per Plant (NPPP), 100 Seed Weight (SW), Seed Yield Per Plant (SYPP) and Days to Pod Initiation (DPI). Among these, EPV and LS are categorical traits, while the rest are numerical. Table \ref{tab:soybeandata} provides a snapshot of the phenotypic data for a few soybean varieties.
	
	\begin{table}[h]
		\centering
			\renewcommand{\arraystretch}{1.3}
		\small 
		\begin{tabular}{|p{0.5cm}|p{0.55cm}|p{0.4cm}|p{0.45cm}|p{1.2cm}|p{0.65cm}|p{0.4cm}|p{0.65cm}|p{0.4cm}|}
			\hline
			\textbf{Sr. No.} & \textbf{EPV} & \textbf{PH} & \textbf{NPB} & \textbf{LS} & \textbf{NPPP} & \textbf{SW} & \textbf{SYPP} & \textbf{DPI}\\
			
			\hline
			1 & Poor	 &54	&6.8	&Moderate	&59.8	&6.5 & 2.5 & 65\\
			\hline
			2 &  Poor	 &67&3.4	&Severe	&33	&6.2 & 3.9 & 64\\
			\hline
			--& -- & -- & -- & -- & -- & -- & --& --\\
			\hline
			2376 & Very Good	&89.6& 5	&Severe	&32.6	&7.3& 3.4&62 \\
			\hline
		\end{tabular}
		\caption{Phenotypic data of soybean plant.}
		\label{tab:soybeandata}
	\end{table}

	In addition to the soybean dataset, we also use a rice dataset obtained from The International Rice Information System (IRIS) (\url{www.iris.irri.org})- a platform for meta-analysis of rice crop plant data. It consists of 12 phenotypic (or physical) characteristics of $1865$ rice species. A snapshot of this data is given in the Table \ref{tab:combinedricedata}.
	
	\begin{table*}[h]
		\centering
			\renewcommand{\arraystretch}{1.3}
		\small 
		\begin{tabular}{|p{0.5cm}|p{1.1cm}|p{1.1cm}|p{1.4cm}|p{0.8cm}|p{0.8cm}|p{1.4cm}|p{1.2cm}|p{1.2cm}|p{0.8cm}|p{0.7cm}|p{1.2cm}|p{0.8cm}|}
			\hline
			\textbf{Sr. No.} & \textbf{Cudicle Reproduction} & \textbf{Cultural Reproduction} & \textbf{Cuneiform Reproduction} & \textbf{Grain Length} & \textbf{Grain Width} & \textbf{Grain weight per 100 seed} & \textbf{HDG 80HEAD} & \textbf{Lightness of Color} & \textbf{Leaf Length} & \textbf{Leaf Width} & \textbf{Plant Post Harvest Traits} & \textbf{Stem Height}\\
			\hline
			1 & 5	&147	 &16	&8.7	&3.1	&2.9	& 102	 &25	&72	&1.1	&29 &54	\\
			\hline
			2 & 6	&150	 &27	&7.1	&3.3	&2.1	& 123	 &20	&73	&1.5	&27 &45	\\
			\hline
			--& -- & -- & -- & -- & -- & -- & -- & -- & -- & -- & -- & --\\
			\hline
			1865 & 3	&56	&16	&7.7	&3.4	&2.8	& 69	&10&	31	&1	&16 &23	\\
			\hline
		\end{tabular}
		\caption{ Phenotypic data of rice plant.}
		\label{tab:combinedricedata}
	\end{table*}

	\subsection{Normalization}\label{subsec:Normalize}
	Let us consider a dataset consisting of $n$ species with $m$ distinct traits. We begin by normalizing the traits as follows \cite{shastri2021probabilistically}:
	\begin{equation}
		(\chi_j){}_i = \frac{(x_j){}_i -\min(x_j)}{\max(x_j)-\min(x_j)}.
	\end{equation}
	Here,$(\chi_j){}_i$ and $(x_j){}_i$ are the normalized and the actual value of the $j^t{}^h$ trait for the $i^t{}^h$ species, respectively. Next, we represent each species as 
	\[
	p_i =
	\begin{bmatrix}
		(\chi_1){}_i \\
		(\chi_2){}_i \\
		\vdots \\
		(\chi_m){}_i \\
	\end{bmatrix},
	\]
	for $i= 1, 2,..,n$.\\
	\vspace{-0.5cm}
	\begin{figure}[h!]
		\includegraphics[width=0.5\textwidth]{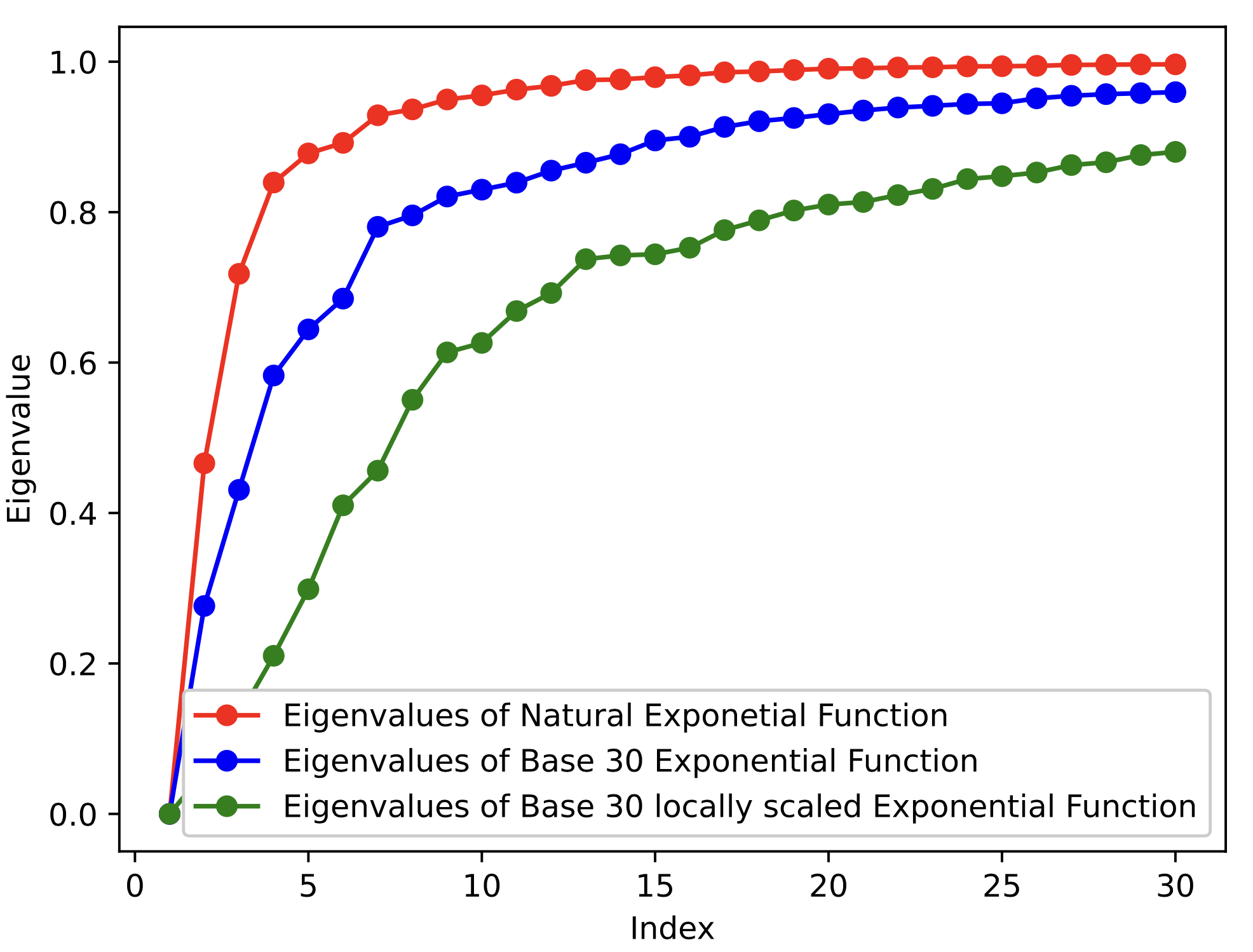}
		\caption{Soybean: First 30 eigenvalues obtained using natural exponential function, base ``$30$'' exponential function, and  base ``$30$'' locally scaled exponential function for building the similarity matrix.}
		\label{SoybeanEigenGap}
	\end{figure}
	
	\subsection{Eigenvalue Analysis}\label{subsec:EigenvalueAnalysis}
	Fig. \ref{SoybeanEigenGap} and \ref{RiceEigenGap} plot the eigenvalues for soybean and rice, respectively. These are first $30$ smallest eigenvalues of the Laplacian matrix obtained from similarity matrix built using natural exponential function, base ``$30$'' exponential function, and  base ``$30$'' locally scaled exponential function.
	
	These figures validate our Conjecture \ref{conjecture1}. That is, the eigenvalues associated with base ``$30$'' exponential function are closer to zero as compared to the eigenvalues associated with the natural exponential function. Thus, as mentioned earlier, using Cheegers's Inequality (\ref{Cheeger}), base ``$30$'' exponential function should result in better clustering than the natural exponential function based clustering. This turns to be true experimentally, which we demonstrate in the results section.  
	
	Second, we further observe that, as claimed in Section \ref{sec:NewestClusteringAlgo}, eigenvalues corresponding to base ``$30$'' locally scaled exponential function are more closer zero than the prior two function choices. Thus, again by using Cheegers's Inequality (\ref{Cheeger}), this function should give the best clustering. This turns to be true experimentally as well, which we again demonstrate in the results section.  
	
	\begin{figure}[h!]
		\includegraphics[width=0.5\textwidth]{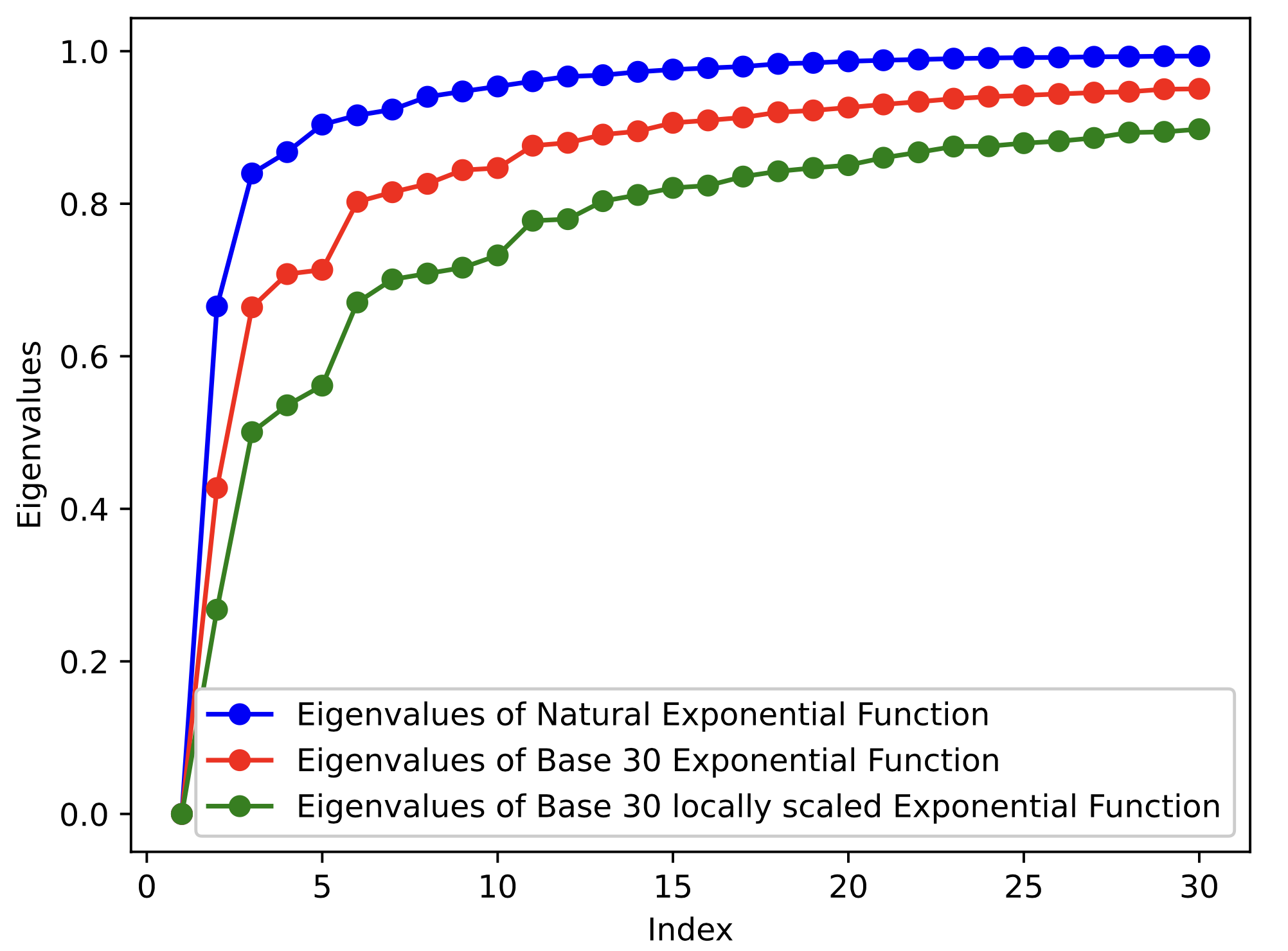}
		\caption{Rice: First 30 eigenvalues obtained using natural exponential function, base ``$30$'' exponential function, and  base ``$30$'' locally scaled exponential function for building the similarity matrix.}
		\label{RiceEigenGap}
	\end{figure}
	\vspace{-0.5cm}
	\section{Results} \label{sec:results}
	
	As discussed in the introduction, we perform experiments on $2376$ soybean and $1865$ rice species. Determining the ideal number of clusters remains an open problem in SC. Based on inputs from plant biologists (which is based on the available number of species of each type) we cluster the soybean data into $10$, $20$, and $30$ groups, and the rice data into $5$, $10$, $15$, and $20$ groups.
	
	To evaluate the quality of clustering we follow the standard definition of Silhoutte Value. This value for data point $p_i$ is given as
	\begin{equation}
		s(p_i)= \frac{b(p_i) -a(p_i)}{\max(a(p_i),b(p_i))},
	\end{equation}
	where $a(p_i)$ denotes the average distance of $p_i$ to the points in its own cluster, while $b(p_i)$ is the average distance of $p_i$ to points in its closest cluster.
		
	\begin{table*}[h]
		\centering
		\small
		\renewcommand{\arraystretch}{1.3}
		\begin{tabular}{|c|c|c|c|c|c|c|c|c|}
			\hline
			&  & Old & Base & New &  & Old SC& Base ``$30$'' & New SC   \\
			Clusters& Distance & SC &  ``$30$'' SC & SC  & HC & Vs  & SC Vs&  Vs   \\
			&  & & & & &  HC \% & HC \% &  HC \%  \\
			\hline

			& Euclidean & 0.2422 &0.2520  & 0.2562&0.2173 & & &    \\ 
			10& SqEuclidean & \textbf{0.3836} & 0.3905  & 0.4066&  \textbf{0.3257}& 17.78&23.42& 33.12 \\ 
			& Correlation & 0.3426 & \textbf{0.4020} & \textbf{0.4336}&0.2307 &  &  &  \\
			\hline

			& Euclidean & 0.2069 & 0.2148&0.0946 & 0.1833&  &&    \\ 
			20& SqEuclidean & \textbf{0.2612} &0.3191& 0.3151 & \textbf{0.2095} &24.69 & 58.8& 63.67 \\ 
			& Correlation & 0.2313&\textbf{0.3327} &\textbf{0.3429} & 0.0598 & & &    \\
			\hline
			
			& Euclidean & \textbf{0.1783} &0.1850 &0.1815 & \textbf{0.1158} & & &    \\ 
			30& SqEuclidean & 0.1538 &0.2588 &\textbf{0.2865} & 0.1086 & 53.97 & 136.01 & 147.41 \\  
			& Correlation & 0.1556 & \textbf{0.2734}&0.2824 & 0.0535 & & &    \\ \hline
			
			\multicolumn{6 }{|c|}{Average percentage gain}       & 32.15  & 72.74 & 81.40        \\ \hline
		\end{tabular}
		\caption{Silhouette values of different clustering algorithms for Soybean.}
		\label{table:without_sampling_soy}
	\end{table*}
	
	\begin{table*}[h]
		\centering
		\small
		\renewcommand{\arraystretch}{1.3}
		\begin{tabular}{|c|c|c|c|c|c|c|c|c|}
			\hline
			&  & Old & Base & New &  & Old SC& Base ``$30$'' & New SC   \\
			Clusters& Distance & SC &  ``$30$'' SC & SC  & HC & Vs  & SC Vs&  Vs   \\
			&  & & & & &  HC \% & HC \% &  HC \%  \\
			\hline
			& Euclidean & 0.1249 & 0.1235&0.1258 & 0.07 & & &    \\ 
			5 & SqEuclidean & 0.215 &0.2214 &0.2242 & 0.1349 & 28.88 & 33.8&34.15  \\ 
			& Correlation & \textbf{0.2200} & \textbf{0.2284}&\textbf{0.2290} & \textbf{0.1707} & & &    \\
			\hline

			& Euclidean & 0.0952 &0.1103 &0.1113 & 0.0662 & & &    \\ 
			10& SqEuclidean & 0.1592 & \textbf{0.1952} & \textbf{0.2002} & 0.1079 & 35.64 &51.55& 55.43 \\ 
			& Correlation & \textbf{0.1747} & 0.1913& 0.1987 & \textbf{0.1288} &  &  &  \\
			\hline

			& Euclidean & 0.0841 & 0.0981&0.0946 & 0.026 &  &&    \\ 
			15& SqEuclidean & 0.1342 &\textbf{0.1714}& \textbf{0.1725} & 0.0677 & 51.55 & 71.23&72.33 \\ 
			& Correlation & \textbf{0.1517}&0.1693 &0.1719 & \textbf{0.1001} & & &    \\
			\hline
			
			& Euclidean & 0.0784 &0.0881 &0.0874 & 0.0128 & & &    \\ 
			20& SqEuclidean & 0.1109 &0.1569 &0.1587 & 0.0432 & 83.35 & 103.15 & 103.40  \\  
			& Correlation & \textbf{0.1454} & \textbf{0.1611}&\textbf{0.1613} & \textbf{0.0793} & & &    \\ \hline
			
			\multicolumn{6 }{|c|}{Average percentage gain}       & 49.86   & 64.93 & 66.33        \\ \hline
		\end{tabular}
		\caption{Silhouette values of different clustering algorithms for Rice.}
		\label{table:without_sampling}
	\end{table*}
	
	Here, we compare four clusterings. First is the standard SC as described in Section \ref{sec:stdSC} (also natural exponential function based SC), and used in \cite{shastri2021probabilistically}. We refer to this as Old SC. Second is our proposed base ``$30$'' exponential function based SC as elaborated in Section \ref{sec:NewClusteringAlgo}. We call this the Base ``$30$'' SC. Third is, again our proposed,  base ``$30$'' locally scaled exponential function based SC as descibed in Section \ref{sec:NewestClusteringAlgo}. We call this New SC. Finally, the fourth is HC, which is mentioned in the literature.
	
	The results of this comparison for soybean and rice are given in Table \ref{table:without_sampling_soy} and \ref{table:without_sampling} respectively. Here, the first column denotes the number of the clusters that are chosen based upon the previous analysis. The second column contains the distance metrics used to build the similarity matrix in the clustering algorithms. Columns three through six list the Silhouette Values of the respective algorithms. Best values in a cell are highlighted in bold. Finally, columns seven through nine give the percentage gain of Old SC, Base ``$30$'' SC, and New SC over HC, respectively. The best values in a cell are used to compute this gain. 
	
	We conclude that the most significant improvement in clustering quality occurs when moving from Old SC to Base ``30” SC with little bit more improvement when going to New SC. That is, for soybean, the gain in these three algorithms over HC is $32.15\%$, $72.74\%$, and $81.40\%$, respectively. In other words, New SC is ${\bf 35\%}$ better than Old SC.
	
	For rice, the gain in these three algorithms over HC is $49.86\%$, $64.93\%$, and  $66.33\%$, respectively. In other words, New SC is ${\bf 11\%}$ better than Old SC. To sum up, New SC yields the best results overall, with soybean showing more significant improvement than rice.	
	
	\section{Conclusion and Future Work} \label{sec:conclusion}
	Phenotypic data of plants is commonly used to group species into different categories, which is further used in breeding programs. Hierarchical Clustering (HC) is a common algorithm that is used for implementing such groupings. Since this algorithm is not very accurate, recently authors in \cite{shastri2021probabilistically} proposed the use of the standard Spectral Clustering (SC) to improve accuracy.  They demonstrated the usefulness of their algorithm via experiments on the soybean plant.
	
	In this work, we propose a novel base ``a'' locally scaled SC that improves the standard SC. {\it First}, using spectral graph theory, specifically the Cheeger’s inequality, we theoretically show that using a base ``a'' exponential function as the similarity function, with increasing base value, leads to improved performance of the SC algorithm. We also integrate our technique with the existing idea of ``local scaling''. {\it Second}, we perform an eigenvalue analysis to support our theoretical result. {\it Third}, using extensive experiments we demonstrate usefulness of our approach. That is,  on $2376$ soybean species and $1865$ rice species, our new algorithm is 35\% and 11\% better than the standard SC, respectively.
	
	\par There are multiple future work directions here. First, in one of the seminal works \cite{ng2001spectral}, the authors have listed sufficiency conditions for SC to work well. It would be very useful to translate those conditions to plant data. Second, it would be useful to experiment with other accurate clusterings (e.g., see \cite{jain2021cube}). Third, although phenotypic characteristics are useful for clustering, genetic data of plant species carries more information. In our earlier work \cite{shastri2019genetic}, we had explored the possibility of using genetic data for clustering and sampling, however, reduced data was used there. It would be interesting to experiment with the full data exhaustively.
	
	Fourth, it would be interesting to improve the existing clustering using mathematical optimisation as in \cite{ahuja2008Mixed} and using approximate computing as in \cite{gupta2020Approx}. Finally, and fifth, implicit relation between phenotypic and genetic data, as done in digital libraries content here \cite{kim2005Explain}, could help in better clustering for both types of data.

\end{document}